\documentclass[11pt]{article}
\usepackage{epsfig} \usepackage{latexsym,bbm}
\usepackage{xspace}
\usepackage{color,fancybox,graphicx,subfigure,fullpage}
\usepackage[top=1in, bottom=1in, left=1in, right=1in]{geometry}
\usepackage{tabularx} 
\usepackage{pdfsync}
\usepackage{enumitem}
\usepackage{microtype}
\usepackage{graphicx}
\usepackage{subfigure}
\usepackage{booktabs}
\usepackage{amsfonts,nicefrac,amsmath,amssymb,xcolor,amsthm,mathtools}
\usepackage{algorithm}
\usepackage{algorithmic}
\usepackage{graphicx}
\usepackage{hyperref}
\usepackage{framed}
\usepackage{float}
\usepackage{multirow}
\usepackage{booktabs}
\usepackage[outdir=./]{epstopdf}
\usepackage{flushend}
\usepackage{olo}
\newtheorem{theorem}{Theorem}

\newtheorem{lemma}{Lemma}

\renewcommand{\epsilon}{\varepsilon}
\usepackage{olo}
\newcommand{\RAN}{\textsc{Ran}\xspace}

\newcommand{\FAIR}{\textsc{Fair-OFUL}\xspace}
\newcommand{\NAIVE}{\textsc{Naive}\xspace }
\newcommand{\EPS}{\textsc{Fair-EPS}\xspace }
\newcommand{\OPT}{\textsc{Opt}\xspace}
\newcommand{\UNC}{\textsc{Unc}\xspace}
\newcommand{\foful}{\textsc{$L_1$-OFUL}\xspace}
\newcommand{\oful}{\textsc{OFUL}\xspace}
\newcommand{\cbone}{\textsc{Confidence-Ball$_1$}\xspace}
\newcommand{\cbtwo}{\textsc{Confidence-Ball$_2$}\xspace}
\newcommand{\epsgreedy}{\textsc{Constrained-$\varepsilon$-Greedy}\xspace}

\title{\bf An Algorithmic Framework to Control Bias in Bandit-based Personalization\footnote{A short version of this paper appeared in the FAT/ML 2017 workshop (\url{https://arxiv.org/abs/1707.02260})}}

\usepackage{authblk}

\author[1]{L. Elisa Celis}
\author[2]{Sayash Kapoor}
\author[3]{Farnood Salehi}
\author[4]{Nisheeth K. Vishnoi}
\affil[1,2,3,4]{\small \'{E}cole Polytechnique F\'{e}d\'{e}rale de Lausanne (EPFL), Switzerland}
\affil[2]{\small Indian Institute of Technology, Kanpur}

\begin{document}

\maketitle
 
\begin{abstract}
Personalization is pervasive in the online space as it leads to higher efficiency and revenue by allowing the most relevant content to be served to each user. However, recent studies suggest that personalization methods can propagate societal or systemic biases and polarize opinions; this has led to calls for regulatory mechanisms and algorithms to combat bias and inequality. Algorithmically, bandit optimization has enjoyed great success in learning user preferences and personalizing content or feeds accordingly. We propose an algorithmic framework that allows for the possibility to control bias or discrimination in such bandit-based personalization. Our model allows for the specification of general fairness constraints on the sensitive types of the content that can be displayed to a user. The challenge, however, is to come up with a scalable and low regret algorithm  for the constrained optimization problem that arises. Our main technical contribution is a provably  fast and low-regret algorithm for the fairness-constrained bandit optimization problem. Our proofs crucially leverage the special structure of our problem. Experiments on synthetic and real-world data sets show that our algorithmic framework can control bias with only a minor loss to revenue.
\end{abstract}
\newpage

\section{Introduction}
Content selection algorithms take data and other information as input, and -- given a user's properties and past behavior -- produce a personalized list of content to display  \cite{goldfarb2011online,liu2010personalized}.
This personalization leads to higher utility and efficiency both for the platform, which can increase revenue by selling targeted advertisements, and also for the user, who sees content more directly related to their interests \cite{Forbes2017,farahat2012effective}.
However, it is now known that such personalization may result in propagating or even creating biases that can influence decisions and opinions. 
Recently, field studies have shown that user opinions about political candidates can be manipulated by personalized rankings of search results \cite{Epstein2015}. 
Concerns have also been raised about gender and racial inequality in serving personalized advertising  \cite{datta2015automated, sweeney2013discrimination, farahat2012effective}.
Just in the US, over two-thirds of adults consume news online on social media sites  \cite{Mitchell2015};
 the  impact of  how social media personalizes content is immense.

One approach to eliminate such biases would be to hide certain user properties so that they cannot be used for personalization; 
however, this could come at a loss to the utility for both the user and the platform --  the content displayed would be less relevant and result in decreased attention from the user and less revenue for the platform 
(see, e.g., \cite{sakulkar2016stochastic}). 
\emph{Can we design personalization algorithms that allow us to be fair without a significant compromise in their utility?}

Here we focus on bandit-based personalization and introduce a rigorous algorithmic approach to this problem.
For concreteness we describe our approach for personalized news feeds, however it also applies to other personalization settings.
Here, content has different {\em types} such as news stories that lean republican vs. democrat, ads for high-paying or low-paying jobs, and is classified into {\em groups}; often based on a single type or a combination of a small number of  types or \emph{sensitive attributes}.
Users can also have different types/contexts and, hence, different preferences over groups, but for simplicity here we focus on the case of a single user.
Current personalization algorithms, at every time-step, select a piece of content for the user,\footnote{In order to create a complete feed, content is simply selected repeatedly to fill the screen as the user scrolls down. The formalization does not change and hence, for clarity, we describe the process of selecting a single piece of content.} and feedback is obtained in the form of whether they click on, purchase or hover over the item.
The goal of the content selection algorithm is to select content for each user in order to maximize the positive feedback (and hence revenue) received.
As this optimal selection is a-priori unknown, the process {is often}  modeled as an online learning problem in which a probability distribution (from which one selects content) is maintained and updated according to the feedback \cite{pandey2006handling}.
As such, as the content selection algorithms learn more about a user, the corresponding probability distributions tend to become sparse (i.e., concentrate the mass on a small subset of entries); the hypothesis is that this is what leads to extreme personalization in which content feeds skew entirely to a single type of content.
To counter this, we introduce a notion of \emph{online group fairness}, in which we require that the probability distribution from which content is sampled  satisfies certain fairness constraints {\em at all time steps}; this  ensures that the probability vectors do not become sparse (or specialize) to a \emph{single} group and thus the content shown to different types of users is not sparse across groups.
Subsequently, we present our model where we do not fix a notion of fairness, as this could depend on the application; instead, our framework allows for the specification of  types or sensitive attributes,  groups, and fairness constraints (in a similar spirit as \cite{yang2016measuring}). 
Unlike previous work, our model can capture general class of constraints; we give a few important examples showing how it can  incorporate prevalent discrimination metrics.
At the same time, the constraints remain linear and allow us to leverage the bandit optimization framework.
While there are several polynomial time algorithms for this setting, the challenge is to come up with a {\em practical} algorithm for this constrained bandit setting. 
Our main technical contribution is to show how an adaptation of an existing algorithm for the unconstrained bandit setting, along with the special structure of our constraints, can lead to a scalable algorithm (with provable guarantees) for the resulting computational problem  of { maximizing revenue while satisfying fairness constraints}.
Finally, we experimentally evaluate our model and algorithms on both synthetic and real-world data sets. 
{Our algorithms approach the theoretical optimum in the constrained setting on both synthetic and real-world data sets and \epsgreedy is very fast in practice.
We study how guaranteeing a fixed amount of fairness with respect to standard fairness metrics (e.g., {\em Risk Difference}) affects revenue on the YOW Dataset \cite{zhang2005bayesian} and a synthetic dataset.
For instance, we observe that for ensuring that the risk difference is less than $1-x$ for $x < \nicefrac{1}{2}$, our algorithms lose roughly $20x\%$ in revenue.
Similarly, to satisfy the 80\% rule often used in regal rulings \cite{DisparateImpactBook}, we lose less than 5\% in revenue in our synthetic experiments.
Our results  show that ensuring fairness is not necessarily at odds with maximizing revenue.}

\begin{figure}[t]
	\begin{center}
		\includegraphics[width=0.6\linewidth]{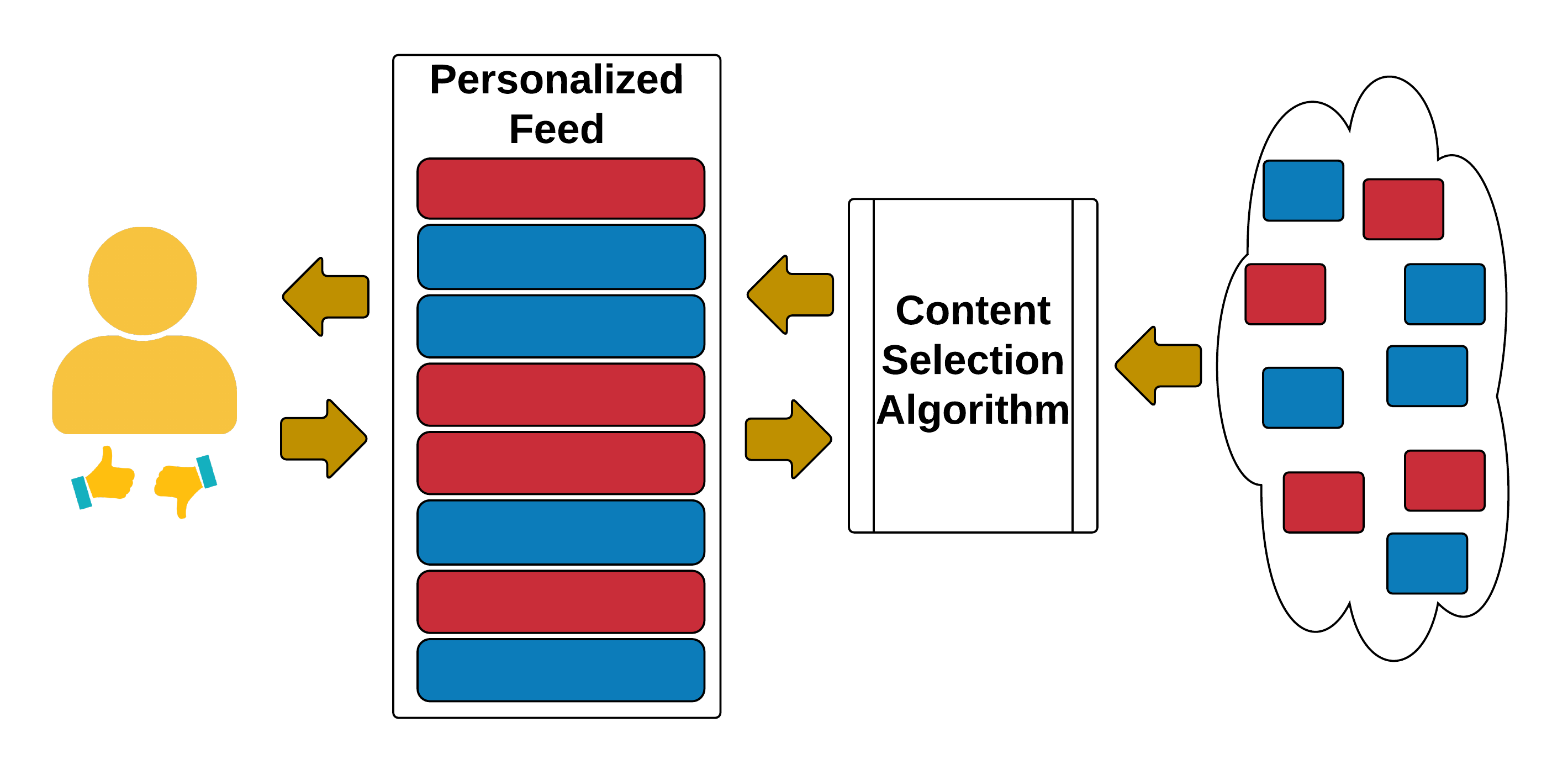}
		\caption{{The content selection algorithm decides what to show to the user 
				and the feedback received from similar users. Different colors represent different types of content, e.g., news stories that lean republican vs. democrat. 
				Feedback could be past likes, purchases or follows.}}
		\label{fig:setting}
	\end{center}
\end{figure}

\begin{figure}
	\centering
	\begin{tabular}{cc}
		\includegraphics[width=0.5\columnwidth]{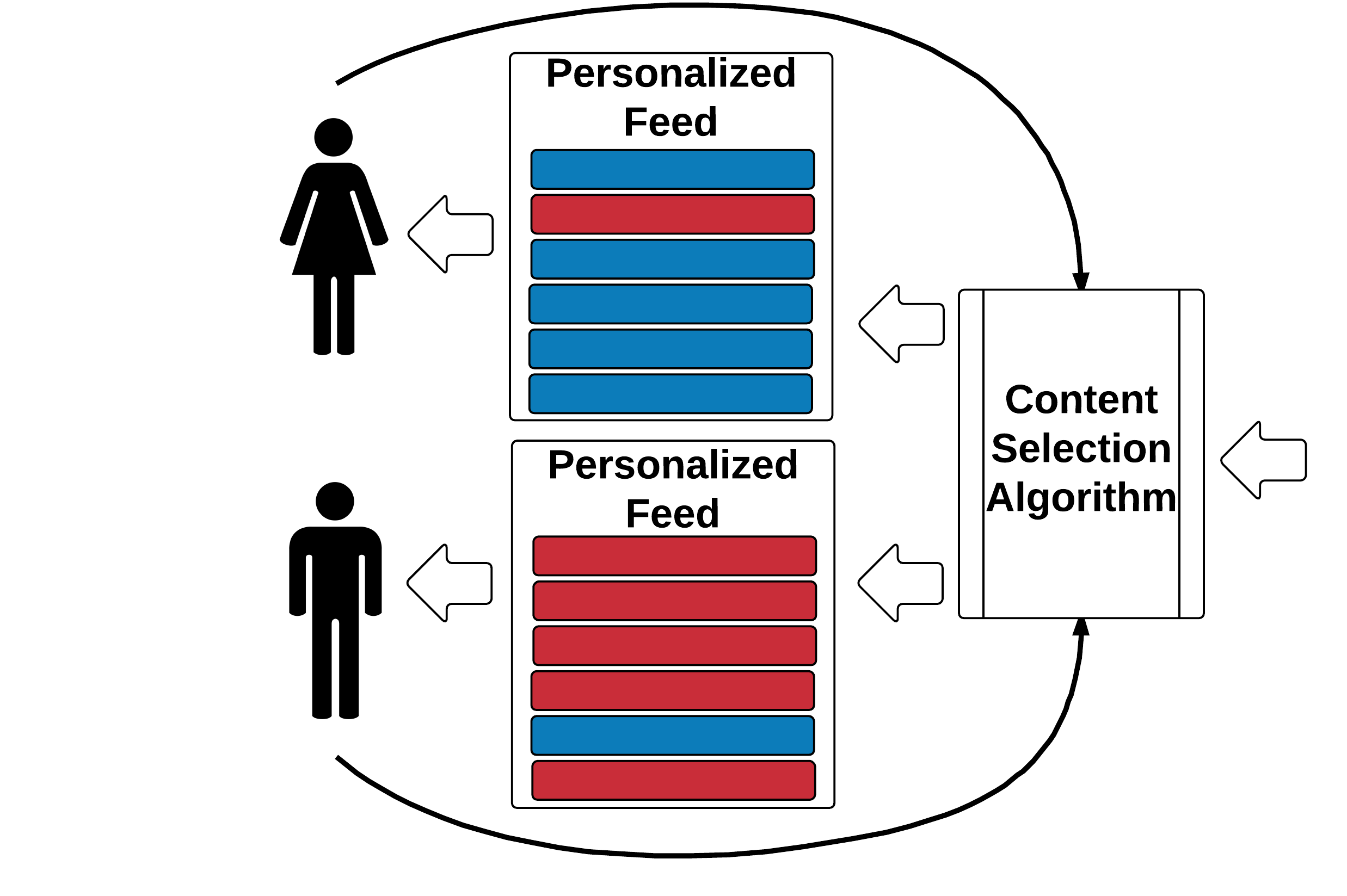}&
		\label{fig:bias}
		\includegraphics[width=0.5\columnwidth]{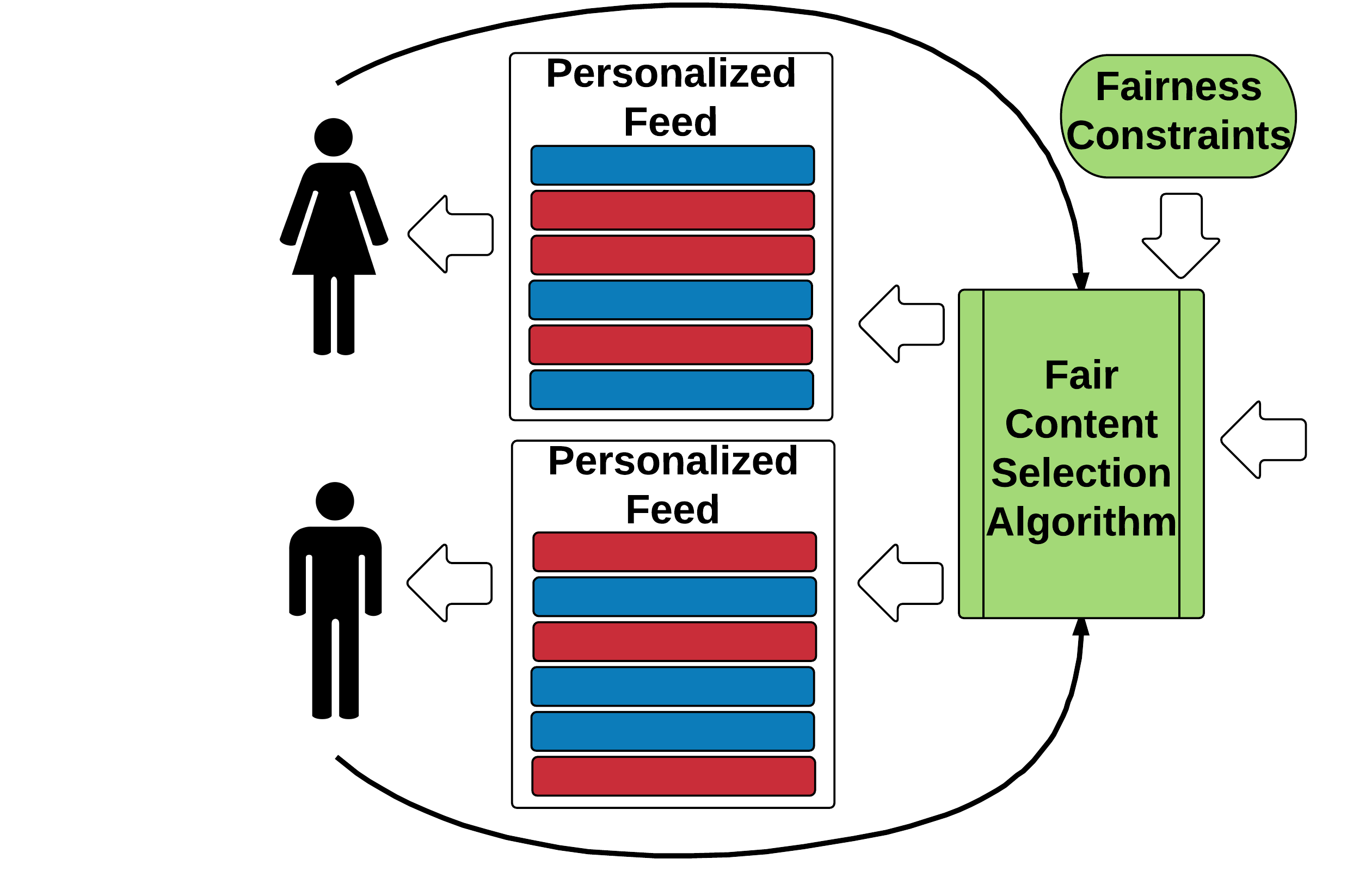}\\
		\label{fig:fair}
		(a)&(b) 	\\
	\end{tabular}
	\caption{ (a) Existing algorithms can perpetuate systemic bias by presenting different types of content to different types of users. (b) Our proposed solution satisfies the fairness constraints and does not allow extreme differences in the types of content presented to different users. Different colors in the content represent different groups, e.g., ads for high vs low paying jobs. Our fair content selection algorithm does not permit extreme biases while personalizing content.}
	\label{fig:bias_fair}
\end{figure}

\section{\bf Bandit Optimization and Personalization}
Algorithms for the general problem of {displaying content} to users largely fall within the multi-armed bandit framework, and \emph{stochastic contextual bandits} in particular \cite{BanditBook}. 
At each time step $t = 1, \ldots, T$, a user views a {page} (e.g., Facebook, Twitter or Google News), the user's  \emph{context} $s^t$ (that lies in a set $\mathcal S$) is given as input, and one {piece of content} (or \emph{arm}) $a^t \in [k]$ must be selected to be displayed. 
A random \emph{reward} $r_{a^t, s^t}$ (hereafter denoted by $r_{a,s}^t$ for readability), which depends on both the given user context and the type of the selected {content} is then received. 
{This reward captures resulting clicks, purchases, or time spent viewing the given content and depends not only on the type of user $s$ (e.g., men may be more likely to click on a sports article) but also on the content $a$ itself (e.g., some news articles have higher quality or appeal than others).}

More formally, at each time step $t$, a sample 
$(s^t, r_{1,s}^t, \ldots, r_{k,s}^t)$ 
is drawn from an unknown distribution $\mathcal D$, the context $s^t \in \mathcal S$ is revealed, the player (the content selection algorithm in this case) selects an arm $a \in [k]$ and receives reward $r_{a,s}^t \in [0,1]$.
As is standard in the literature, we assume that the $r_{a,s}$s are drawn independently across $a$ and $t$ (and not necessarily $s$).  
The rewards $r_{a^\prime,s^\prime}$ for any $a^\prime \neq a$ and $s^\prime \neq s$ are assumed to be {\em unknown} -- indeed, there is no way to observe what a user's actions {would have been} had a different {piece of content} been displayed, or what a different user would have done.

The algorithm computes a probability distribution $p^{t}$ over the {arms} based on the previous observations $(s^1, a^1, r_{a,s}^1), \ldots, (s^{t-1},a^{t-1},r_{a,s}^{t-1})$ and the current user type $s^{t}$, and then selects {arm} $a^t \sim p^t$; as $p^t$ depends on the context $s^t$, we often write $p^t(s^t)$ for clarity. 
The goal is to select $p^t(s^t)$s in order to maximize the cumulative rewards, and 
the efficacy of such an algorithm is measured with respect to how well it minimizes \emph{regret} -- the difference between the algorithm's reward and the reward obtained from the (unknown) optimal policy.
Formally, let $f: \mathcal S \to [k]$ be a mapping from contexts to {arms}, and let $f^\star := \arg\max_{f} \mathbb E_{(s,\vec{r})\sim\mathcal D} [r_{f(s),s}];$ i.e., $f^\star$ is the policy that selects the best arm in expectation for each context. 
Then, the regret is defined as
$  \mathsf{Regret}_T := T \cdot \mathbb E_{(s,\vec{r})\sim\mathcal D} [r_{f^\star(s),s}] - \sum_{t=1}^T r_{a,s}^t. $
Note that the regret is a random variable as $a^t$ depends not only the draws from $p^t$, but also on the realized \emph{history} of samples $ \{(s^t, a^t, r_{a,s}^t)\}_{t=1}^T$. 

\section{Our Model}
We would like a model that can guarantee fairness with respect to sensitive attributes of content in the bandit framework used for personalization.
Guaranteeing such {\em group fairness} would involve controlling disproportionate representation  across the sensitive attributes. 
Towards defining our notion of group fairness,  
 let $G_1,$ $\ldots,$ $G_g \subseteq [k]$ be $g$ \emph{groups} of {content}. 
{For instance the $G_i$s could be a partition (e.g., ``republican-leaning'' news articles, ``democratic-leaning'', and ``neutral'').
	An important feature of bandit algorithms, which ends up being the root cause of  bias,  is that the probability distribution converges to the action with the best expected reward for each context; i.e., the entire probability mass in each context ends up on a single group. 
	This leads to an effect where different users may be shown very different ad groups
	and can be problematic when it leads to observed outcomes such as only showing minimum-wage jobs to disenfranchised populations.} 
In Section \ref{sec:metrics-app}, we show that many metrics for discrimination or bias  boil down to quantifying how different these probability distributions across groups can be. 
Thus, finding a mechanism to control these probability distributions would  ensure fairness with respect to many metrics. 

\begin{table}
\centering
	\begin{tabular}{ |c|c|c| } 
		 \hline
		 \textbf{Algorithm} & \textbf{Per iteration Running time} & \textbf{Regret Bound}\\ 
		 \hline
		 \cbtwo \cite{dani2008stochastic} & NP-Hard problem  & $O\left(\frac{k^2}{\gamma} \log^3 T\right)$ \\ 
		 \hline
		 \oful \cite{YA2011} & NP-Hard problem  & $\tilde{O}\left(\frac{1}{\gamma}\left(k^2+\log^2T\right)\right)$\\ 
		 \hline
		 \cbone \cite{dani2008stochastic} & $O\left(k^\omega\right) + 2k$ LP-s & $O\left(\frac{k^3}{\gamma} \log^3 T\right)$\\ 
		 \hline
		 \foful (Algorithm \ref{algo:foful}) & $O\left(k^\omega\right) + 2k$ LP-s & $\tilde{O}\left(\frac{k}{\gamma}\left(k^2+\log^2T\right)\right)$ \\
		 \hline
		  \epsgreedy (Algorithm \ref{algo:epsgreedy}) & $O\left(1\right) + 1$ LP  & $O\left(\frac{k}{\gamma^2} \log T\right)$  \\ 
		 \hline
	\end{tabular}
	\caption{The complexity and problem-dependent regret bounds for various algorithms when the decision set is a polytope. }
\label{table:comparison}
\end{table}
This motivates our definition of \emph{group fairness constraints}.
For each group $G_i$, let $\ell_i$ be a lower bound and $u_i$ be an upper bound  on the amount of probability mass that the content selection algorithm can place on this group. Formally, we impose 
the following constraints:
\begin{equation}
	\label{eq:fair}
	\ell_i \leq \sum_{a \in G_i} p^t_a(s) \leq   u_i \;\;\;\; \forall i \in [g], \forall t \in [T], \forall s \in \mathcal S.
\end{equation}
The bounds $\ell_i$s and $u_i$s provide a handle with which we can ensure that the probability mass placed on any given group is neither too high nor too low at each time step.

Rather than fixing the values of $u_i$s and $\ell_i$s, we allow them to be specified as input.
This allows one to control the extent of group fairness depending on the application, and hence (indirectly) encode bounds on a wide variety of existing fairness metrics.
This typically involves translating the fairness metric parameters into concrete values of $\ell_i$s and $u_i$s; see  Section \ref{sec:metrics-app} for several examples.
For instance, given $\beta>0$, by setting $u_i$s and $\ell_i$s such that  $u_i-\ell_i\leq \beta$ for all $i$, we can ensure that, what is referred to as, the {\em risk difference} is bounded by $\beta$. 
An additional feature of our model is that no matter what the group structures, or the lower and upper bounds are, the constraints are always convex.

Importantly, note that unlike ignoring user contexts entirely, the constraints still allow for personalization \emph{across} groups. 
For instance, if the groups are republican (R) vs democrat (D) articles, and the user contexts are known republicans (r) or democrats (d), we may require that $p^t_{\mbox{R}}(\cdot) \leq 0.75$ and $p^t_{\mbox{D}}(\cdot) \leq 0.75$ for all $t$. 
This ensures that extreme polarization cannot occur -- at least 25\% of the articles a republican is presented with will be democrat-leaning.
Despite these constraints, personalization at the group level can still occur, e.g., by letting $p_{\mbox{R}}^t(\mbox{r}) = 0.75$ and $p_{\mbox{R}}^t(\mbox{d}) = 0.25$.
Furthermore, this framework allows for complete personalization \emph{within} a group; e.g., the republican-leaning articles shown to republicans and democrats may differ. 
This is crucial as the utility maximizing republican articles for a republican may differ than the utility maximizing republican articles for a democrat.  
To the best of our knowledge, such fairness constraints are novel for personalization. 
Here the constraints allow us to addresses the concerns illustrated in our motivating examples. 

The next question we address is how to  measure an algorithm's performance against the best \emph{fair} solution.
\footnote{The unconstrained regret may be arbitrarily bad, e.g., if $u_i = \varepsilon \ll 1$ for the group $i$ that contains the arm with the best reward.} 
We say that a probability distribution $p$ on $[k]$ is \emph{fair} if it satisfies the upper and lower bound constraints in \eqref{eq:fair}, and let $\mathcal C$ be the set of all such probability distributions. Note that given the linear nature of the constraints, the set $\cC$ is a polytope.
Let $\mathcal B$ be the set of functions $g: \mathcal S \to [0,1]^k$ such that $g(s) \in \mathcal C$; i.e., all $g \in \mathcal B$ satisfy the fairness constraints.  
Further, we let 
$g^\star := \arg\max_{g \in \mathcal B} \mathbb E_{(s,\vec{r})\sim\mathcal D} [r_{g(s),s}];$ i.e., $g^\star$ is the policy that selects the best arm in expectation for each context. 
An algorithm is said to be fair if it only selects $p^t(s^t) \in \mathcal C$.
Thus, the \emph{fair regret} for such an algorithm can be defined as
$  \mathsf{FairRegret}_T := T \cdot \mathbb E_{(s,\vec{r})\sim\mathcal D} [r_{g^\star(s),s}] - \sum_{t=1}^T r_{a^t,s^t}.$

\section{Other Related Work}
Studies considering notions of group fairness such as {statistical parity},  {disparate impact}, and others mentioned above (see, e.g., \cite{kamiran2009classifying, feldman2015certifying}), apply these metrics the \emph{offline} problem; in our setting this would correspond to enforcing $p^T(s)$ to be roughly the same for all $s$, but would leave the intermediary $p^t(s)$ for $t < T$ unrestricted. 
A subtle point is that most notions of (offline) group fairness primarily consider the selection or classification of groups of \emph{users}; in our context, while the goal is still fairness towards users, the selection is over \emph{content}. 
However, as the  constraints necessary to attain fairness remain on the selection process, we use the terminology  \emph{online group fairness} to highlight these parallels.

In a completely different bandit setting, a recent work \cite{joseph2016fairness} defined a notion of \emph{online individual fairness}  which, in our language, restricts $p^t$s so that all \emph{arms} are treated equally by only allowing the probability of one arm to be more than another if we are reasonably certain that it is better than the other.
When the arms correspond to \emph{users} and not content such individual fairness is indeed important, but for the personalized setting the requirement is both too strong (we are only concerned with \emph{groups} of content) and too weak (it still allows for convergence to different groups for different contexts). 
Other constrained bandit settings that encode \emph{global} knapsack-like constraints (and locally place no restriction on $p^t$) have also been considered; see, e.g., \cite{agrawal2016linear}. 
Our two fair bandit algorithms  build on the calssic work of \cite{auer2002finite} and works on linear bandit optimization over convex sets \cite{dani2008stochastic,YA2011}.

\section{Algorithmic Results}

We present two  algorithms that attempt to minimize regret in the presence of  fairness constraints.
Recall that our fairness constraints are linear and the reward function linear.

For each arm $a \in [k]$ and each context $s \in \mathcal{S}$, let its mean reward be $\mu^\star_{a,s}$. 
Our algorithms do not assume any relation between different contexts and, hence, function independently for each context.  
Thus, we describe them for the case of a single fixed context.
In this case, the unknown parameters are the expectations of each arm $\mu^\star_a$ for $a \in [k]$.
We assume that the reward for the $t$-th time step is sampled from a Bernoulli distribution with probability of success $\mu^\star_{a^t}.$ 
In fact, the reward can be sampled from any bounded distribution with the above mean; we explain this further in Section \ref{sec:proofs_fofulreg}.
For this setting, we present two different algorithms,  \foful (Algorithm \ref{algo:foful}) and \epsgreedy (Algorithm \ref{algo:epsgreedy}): the first has a better regret, and the second a better running time. The latter fact, as we discuss a bit later, is  due to the special structure arising in  our model when compared to the general linear bandit setting.

\begin{theorem}
	Given the description of $\mathcal{C}$ and the sequence of rewards,  the \foful algorithm, run for $T$ iterations,  has the following fair regret bound for each context $s \in \cS$:
	$\mathbb{E}\left[{\sf FairRegret}_T\right]=O\left(\frac{k}{\gamma} \left(\log^2T + k\log T + k^2\log\log T \right)\right),$
	where the expectation is taken over the histories and $a^t \sim p^t$, and $\gamma$ depends on $(\mu^\star_{a})_{a \in [k]}$ and $\mathcal{C}$ as defined in \eqref{eq:gamma}.
	\label{theorem:fofulreg}
\end{theorem}

\begin{theorem} \label{thm:epsgreedy}
	Given the description of $\cC$, a fair probability distribution $q_f \in \left\{q : B_{\infty}(q,\eta) \subset \mathcal{C} \right\}$,
	and the sequence of rewards,  the $\epsgreedy$ algorithm, run for $T$ iterations, has the following fair regret bound for each context $s \in \mathcal{S}$:
	$\mathbb{E}\left[{\sf FairRegret}_T\right]  = \nonumber  O 
	\left(   \frac{\log T}{\eta \gamma^2}   \right),
$	
	where $\epsilon_t = \min\{1,\nicefrac{4}{(\eta d^2t)}\}$ and $d = \min\{\gamma, \nicefrac{1}{2}\}$.\footnote{$B_\infty(q,\eta)$ is an $\ell_\infty$-ball of radius $\eta$ centered at $q$. The algorithm works for any lower bound $L$ on $\gamma$, with a $L$ instead of $\gamma$ in the regret bound.} 
	\label{theorem:epsgreedy}
\end{theorem}
\noindent
The quantity $\gamma$ is the difference between the maximum and the second maximum of the expected reward with respect to the $\mu^\star$s over the vertices of the polytope $\mathcal{C}$.
Formally, let $V(\mathcal{C})$ denote the set of vertices of $\mathcal{C}$ and 
$v^\star := \arg\max_{v \in V(\mathcal{C})} \sum_{a \in [k]} \mu^\star_a v_a.$ Then,
\begin{equation}\label{eq:gamma}
 \gamma:=  \sum_{a \in [k]} \mu^\star_a v_a^\star - \max_{v \in V(\mathcal{C}) \backslash v^\star} \sum_{a \in [k]} \mu^\star_a v_a. 
\end{equation}
For general convex sets, $\gamma$ can be $0$ and the regret bound can at best only be $\sqrt{T}$ \cite{dani2008stochastic}.
As our fairness constraints result in a $\mathcal{C}$ is a polytope, unless there are degeneracies, $\gamma$ is non-zero. 
In general, $\gamma$ may be hard to estimate theoretically.  
However, for the settings in which we conduct our experiments on, we observe that the value of $\gamma$ is reasonably large.

In traditional algorithms for Multi-Armed bandits, when the probability space is unconstrained, it suffices to solve $\argmax_{i \in [k]} \tilde{\mu}_i$, where $\tilde{\mu}_i$ is an estimate for the mean reward of the $i$-th arm. 
It can be an optimistic estimate for the arm mean in case of the UCB algorithm \cite{auer2002finite}, a sample drawn from the normal distribution with the mean set as the empirical mean for the Thompson Sampling algorithm \cite{agrawal2012analysis} etc.
When the probability distribution is constrained to lie in a polytope $\cC$, instead of a maximum over the arm mean estimates, we need to solve $\argmax_{p \in \cC} \tilde{\mu}^\top p$.
This necessitates the use of a linear program for any algorithm operating in this fashion. 
At every iteration, \foful solves $2k$ LP-s, and \epsgreedy solves one LP. 

\epsgreedy thus offers major improvements in running time over \foful. 
The regret bound of Algorithm \ref{algo:epsgreedy} has better dependence on $k$ and $T$, but is worse by a factor of $\nicefrac{1}{\gamma}$, as compared to Algorithm \ref{algo:foful}; see {Table \ref{table:comparison}}.  
We now give overviews of the both the algorithms. 
The full proofs of Theorems \ref{theorem:fofulreg} and \ref{theorem:epsgreedy} appear in Sections \ref{sec:proofs_fofulreg} and \ref{sec:proofs_epsgreedy} respectively.

\begin{figure}
	\begin{algorithm}[H]
		\caption{ \foful}
		\label{algo:foful}
		\begin{algorithmic}[1]
			{
				\REQUIRE Constraint set $\cC$, maximum failure probabilty $\delta$, an $L_2$-norm bound on $\mu^\star$: $\norm{\mu^\star}_2 \leq \sigma$ and a positive integer $T$
				\STATE Initialize $V_1 := \cI,$  $\hat{\mu}_1 := 0,$ and 
				 $b_1 := 0$
				\FOR{$t = 1, \ldots, T$} 
				\STATE Compute {${\beta_t(\delta)} := \left(\sqrt{2\log\left(\frac{\det(V_t)}{\delta}\right)^\frac{1}{2}} + \sigma \right)^2$}
				\STATE Denote {$B_t^1 := \left\{\mu : \norm{\mu - \hat{\mu}_t}_{1, V_t} \leq \sqrt{k\beta_t(\delta)}\right\}$}
				\STATE Compute {$p^t := \argmax_{p \in \cC} \max_{\mu \in B_t^1} \mu^\top p$}
				\STATE Sample $a$ from the probability distribution $p^t$
				\STATE {Observe reward $r_t = r^t_{a}$}
				\STATE Update {$V_{t+1} := V_t + p^t{p^t}^\top$ }
				\STATE Update {$b_{t+1} := b_t + r_tp^t$}
				\STATE Update {$\hat{\mu}_{t+1} := V_{t+1}^{-1}b_{t+1}$}
				\ENDFOR
			}
		\end{algorithmic}
	\end{algorithm}
\end{figure}
\paragraph{\textsc{$\boldsymbol{L_1}$-OFUL}\xspace.}
At any given time $t$, \foful maintains a regularized least-squares estimate for the optimal reward vector $\mu^\star$, which is denoted by $\hat{\mu}_t$. 
At each time step $t$, the algorithm first constructs a suitable confidence set $B_t^1$ around $\hat{\mu}_t$.
Roughly, the definition of this set ensures that the confidence ball is ``flatter'' in the directions already explored by the algorithm so it has more likelihood of picking a probability vector from unexplored directions.
The algorithm chooses a probability distribution $p^t$ by solving a linear program on each of the $2k$ vertices of this confidence set, and plays an arm $a^t \sim p^t$. 
Recall that for each arm $a\in[k]$, the mean reward is $\mu^\star_a\in[0,1]$. 
The reward for each time step is generated as $r_t\sim\text{Bernoulli}(\mu^\star_{a^t})$, where $a^t \sim p^t$ is the arm the algorithm chooses at the $t^{th}$ time instant. 
The algorithm observes this reward and updates its estimate to $\hat{\mu}_{t+1}$ for the next time-step appropriately.

\foful (Algorithm \ref{algo:foful}) is an adaptation of the \oful algorithm that appeared in \cite{YA2011}.
The key difference is that instead of using a scaled $L_2$-ball in each iteration,  we use a a scaled $L_1$-ball (Step 4 in Algorithm \ref{algo:foful}). 
As we explain below, this makes  Step 5 of our algorithm efficient as opposed to that of \cite{YA2011} where the equivalent step required solving a NP-hard and nonconvex optimization problem.
This idea is similar to how \cbtwo was adapted to \cbone in \cite{dani2008stochastic}.
In particular, our algorithm improves, by a multiplicative factor of $O\left(\log T\right)$, the regret bound of $\left(O\left(\frac{k^3}{\gamma} \log^3 T\right)\right)$ of \cbone in \cite{dani2008stochastic}, see Table \ref{table:comparison}. 

We show that the \foful algorithm can be implemented in time polynomial in $k$ at each iteration.
Apart from Step 4, where we need to find
$\argmax_{p \in \cC}\max_{\mu \in B_t^1}\mu^\top p,
$
the other steps are quite easy to implement efficiently.
For Step 4, we assume oracle access to a linear programming algorithm which can efficiently (in poly($k$) time) compute $\argmax_{p \in \cC}\mu^\top p$. (where $\mu$ is the input to the oracle). 
We first change the order of the maximization, i.e., $\max_{p \in \cC} \max_{\mu \in B_t^1} \mu^\top p = \max_{\mu \in B_t^1} \max_{p \in \cC} \mu^\top p$. 
Using the linear programming oracle, we can solve the inner maximization problem of finding $\max_{p \in \cC} \mu^\top p$ for any given value of $\mu$. 
It is enough to solve this at the $2k$ vertices of $B_t^1$ and take the maximum of these as our value of $\max_{p \in \cC} \max_{\mu \in B_t^1}\mu^\top p$, since one of these $2k$ vertices would be the maximum. 
The value of $p$ corresponding to this maximum value would be the required value $\argmax_{p \in \cC}\max_{\mu \in B_t^1}\mu^\top p$.
Thus, in $2k$ calls to this oracle, we can find the desired probability distribution $p^t$.

Note that in order to find the value of  $\det V_{t+1}$ and $V_{t+1}^{-1}$ in Steps 3 and 10 of the algorithm, we can perform rank-one updates by using  the well-known Sherman-Morrison formula, which bring down the complexity for these steps to $O(k^2)$ from $O(k^3)$, since we already know the value of $\det V_{t}$ and $V_{t}^{-1}$.

\begin{figure}
	\begin{algorithm}[H]
		\caption{ \epsgreedy}
		\label{algo:epsgreedy}
		\begin{algorithmic}[1]
			{
				\REQUIRE Constraint set $\cC$, a fair probability distribution $q_f \in \left\{q : B_{\infty}(q,\eta) \subset \mathcal{C} \right\}$, a positive integer $T$, a constant $L$ that controls the exploration
				\STATE Initialize $\bar{\mu}_1 := 0$
				\FOR{$t = 1, \ldots, T$} 
				\STATE Update $\epsilon_t := \min\{1,\nicefrac{4}{(\eta L^2t)}\}$
				\STATE Compute {$p^t := \argmax_{p \in \cC} \bar{\mu}_t^\top p$}
				\STATE Sample $a$ from the probability distribution $(1-\epsilon_t)p^t + \epsilon_t q_f$
				\STATE {Observe reward $r_t = r^t_{a}$}
				\STATE Update empirical mean {$\bar{\mu}_{t+1}$}
				\ENDFOR
			}
		\end{algorithmic}
	\end{algorithm}
\end{figure}

\paragraph{\textsc{Constrained-$\boldsymbol{\varepsilon}$-Greedy}\xspace.} 
Compared to \foful, instead of maintaining a least-squares estimate of the optimal reward vector $\mu^\star$, {\epsgreedy} maintains an empirical mean estimate of it denoted by $\bar{\mu}_t$.
The algorithm, with probability $1-\eps$ chooses the probability distribution $p^t = \argmax_{p \in \cC}\bar{\mu}^\top p $, and with probability $\eps$ it samples from a feasible fair distribution $q_f \in \mathcal{C}$ in the $\eta$-interior.
It then plays an arm $a^t \sim p^t$. 
The reward for each time step is generated as $r_t\sim\text{Bernoulli}(\mu^\star_{a^t})$, where $a^t \sim p^t$ is the arm the algorithm chooses at the $t^{th}$ time instant. 
The algorithm observes this reward and updates its estimate to $\bar{\mu}_{t+1}$ for the next time-step appropriately.
Maintaining an empirical mean estimate instead of a least-squares estimate, and solving only one linear program instead of $2k$ linear programs at every iteration causes the main decrease in running time compared to \foful.

\epsgreedy is a variant of the classical $\epsilon$-Greedy approach \cite{auer2002finite}. 
Recall that in our setting, an arm is an ad (corner of the $k$-dimensional simplex) and not a vertex of the polytope $\mathcal{C}$.
The polytope $\mathcal{C}$ sits inside this simplex and may have exponentially many vertices.
This is not that case in the setting of \cite{dani2008stochastic,YA2011} -- there may not be any ambient simplex in which their polytope sits, and even if there is, they do not use this additional information about which vertex of the simplex was chosen at each time $t$.
Thus, while they are forced to maintain confidence intervals of rewards for all the points in $\mathcal{C}$, this speciality in our model allows us to get away by maintaining confidence intervals only for the $k$ arms (vertices of the simplex) and then use these intervals to obtain a confidence interval for any point in $\mathcal{C}$.  
Similar to $\epsilon$-Greedy, if we choose each arm enough number of times, we can build a good confidence interval around the mean of the reward for each arm.  
The difference is that instead of converging to the optimal arm, our constraints maintain the point inside $\mathcal{C}$ and it converges to a vertex of $\mathcal{C}$.

To ensure that we choose each arm with high probability, we fix a fair point $q_f \in  \eta$-interior of $\mathcal{C}$ and sample from the point $(1-\eps)p^t + \eps q_f$.
  Then, as in $\epsilon$-Greedy,  we proceed by bounding the regret showing that if the confidence-interval is tight enough, the optimal of LP with true mean $\mu^\star$ and LP with the empirical mean $\bar{\mu}$ does not change. 

\textbf{Solving the LP.} In both \foful and \epsgreedy,  if the ``groups'' in the constraint set form a  partition, one can solve the linear program in $O(k)$ time via a simple greedy algorithm. 
This is because, since each part is separate, the decision of how much probability mass goes to which group can be decided by a simple greedy process and, once that is decided, how to distribute the probability mass within each group can be also done trivially using a greedy strategy. 
This can be extended to {\em laminar} family of constraints and we can solve the LP step in $O(gk)$ time exactly. 
We provide more details in Section \ref{sec:laminar}.
For general group structures, given that the constraints are of packing/covering type, we believe that the  algorithm of \cite{AllenZhuOrecchia} may be useful to obtain a fast and approximate LP solver. 
\begin{figure*}[t!]
	\centering
	\begin{tabular}{ccc}
		{\includegraphics[width=0.45\textwidth]{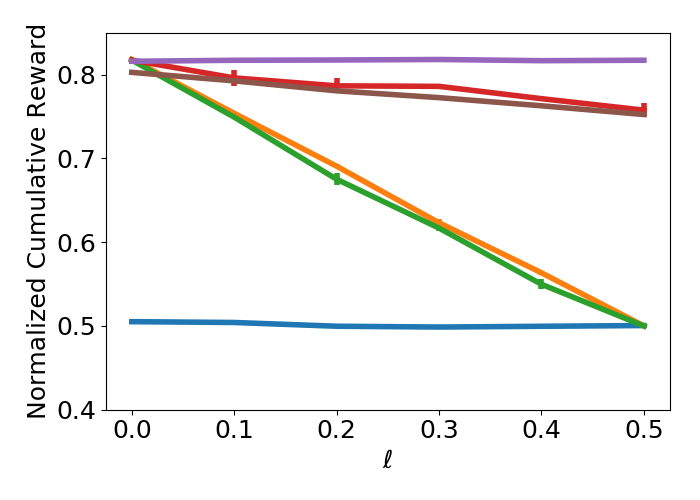}}&
		{\includegraphics[width=0.45\textwidth]{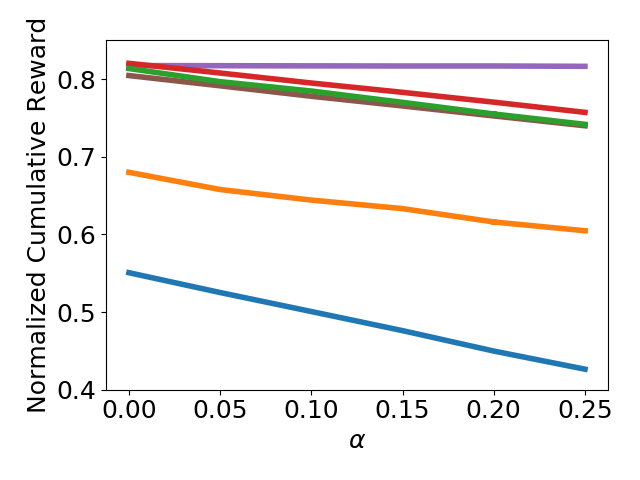}}&{\includegraphics[width=0.1\textwidth,trim = 0cm -4cm 0cm 0cm]{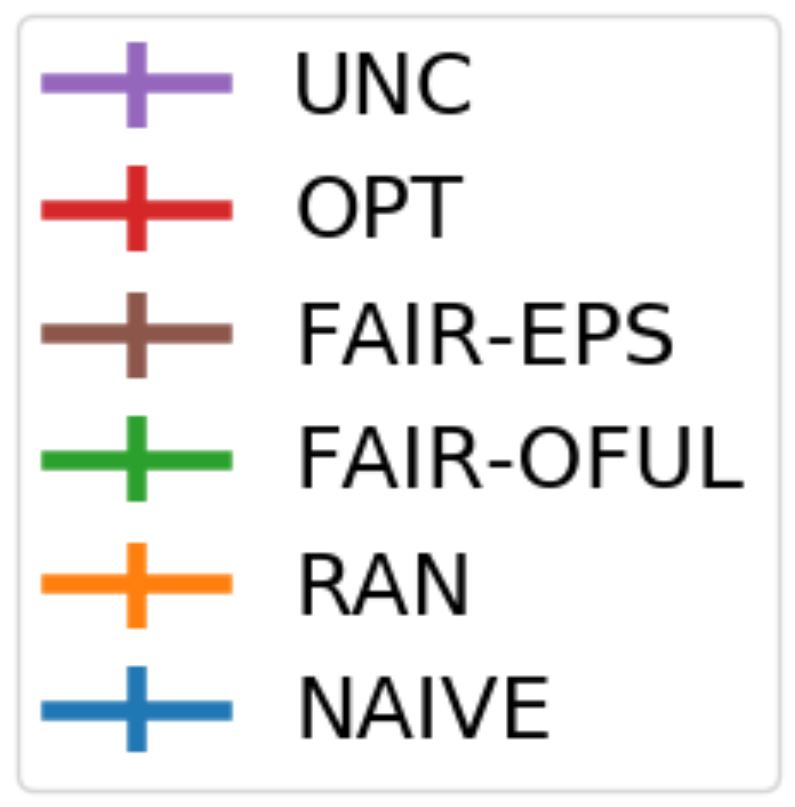}}
		\\ 
		$\qquad \; \;$(a) & $\qquad \;\;$(b)
	\end{tabular}
	\caption{{Empirical results on Synthetic Data.} Depicts the normalized cumulative reward of the algorithms we consider 
		(a) as we vary the lower bound constraints $\ell_1 = \ell_2 = \ell$, and (b) as we vary the reward $\alpha$ lost when presenting content from a group the user does not prefer.
	}
	\label{fig:hm}
\end{figure*}
\section{Empirical Results}
\label{sec:experiments}
In this section we describe experimental evaluations of our algorithms when run on various constraint settings of our model on rewards derived both synthetically and from real-world data.
In each experiment we report the normalized cumulative reward for each of the following algorithms and benchmarks:
\begin{itemize}[leftmargin=*]
	\item {\NAIVE}. As a baseline, we consider a simple algorithm that satisfies the constraints as follows: For each group $i$, with probability $\ell_i$ it selects an arm uniformly at random from $G_i$, then, with any remaining probability, it selects an arm uniformly at random from the entire collection $[k]$ while respecting the upper bound constraints $u_i$.
	
	\item {\UNC}. For comparison, we also depict the performance of the \emph{unconstrained} \foful algorithm (in which we let $\mathcal{C}$ be the set of {\em all} probability distributions over $[k]$).
	
	\item {\RAN}.  At each time step, given the probability distribution $p^t$ specified by the unconstrained \foful algorithm, take the largest $\theta \in [0,1]$ such that selecting an arm with probability $\theta \cdot p^t$ does not violate the fairness constraints. With the remaining probability $(1-\theta)$ it follows the same procedure as in \NAIVE to select an arm at \emph{random} subject to the fairness constraints.
	
	\item {\FAIR}. Our implementation of \foful with the given fairness constraints as input. 

	\item {\EPS}. Our implementation of \epsgreedy with the given fairness constraints as input\footnote{We set $\epsilon_t = \min(1, \nicefrac{10}{t})$. Tuning $\epsilon_t$ might give better results, depending on the dataset used.}.
	
	\item {\OPT}. For comparison, we often depict the performance of the hypothetical \emph{optimal} probability distribution, subject to the fairness constraints, that we could have used had we known the reward vector {$\mu^\star$} for the arms a-priori. Note that this optimal distribution is easy to compute via a simple greedy algorithm; it simply places the most probability mass that satisfies the constraints on the best arm, the most probability mass remaining on the second-best arm subject to the constraints, and so on and so forth until the entire probability mass is exhausted.
\end{itemize}
Note that \FAIR, \RAN, and \UNC  all use \foful as a subroutine; however only \FAIR and \EPS take in the constraints as input, with \RAN satisfying the fairness constraints via a different approach, and \UNC does not satisfy the fairness constraints at all.

\subsection{Experiments on Synthetic Data}

\paragraph{Synthetic Data.} 
We first consider a simple synthetic arm model in order to illustrate the tradeoffs between fairness, rewards, and arm structure. 
We consider 2 groups, each containing 4 arms that give Bernoulli rewards. When a user prefers one group over another, we decrease the rewards of the group they don't like by a fixed value $\alpha$. 
In the experiments, we let the mean of the rewards be $\left[0.28, 0.46, 0.64, 0.82\right]$, and let $\alpha = .1$ unless specified otherwise.\footnote{These values were chosen because they are the expected values of the arm reward, from smallest to largest, of 4 arms are sampled from $\cU[\alpha,1]$ where $\alpha = .1$.}
In each experiment we perform 100 repetitions and report the normalized cumulative reward after 1000 iterations; error bars represent the standard error of the mean.

\paragraph{Effect of Fairness Constraints.}
As there are only two groups, setting a lower bound constraint $\ell_1 = \zeta$ is equivalent to setting an upper bound constraint $u_2 = 1-\zeta$. Hence, it suffices to see the effect as we vary the lower bounds. 
We fix $\alpha = .1$, and vary $\ell_1 = \ell_2 = \ell$ from {0 to .5}; i.e., a completely unconstrained setting to a fully constrained one in which each group has exactly 50\% probability of being selected.

We observe that, even for very small values of $\ell$, the $\FAIR$ and $\EPS$ algorithms significantly outperform \RAN.  Indeed, the performance of the $\FAIR$ and $\EPS$ algorithms is effectively the same as the performance of the (unattainable) hypothetical optimum, and is only worse than the unconstrained (and hence unfair) algorithm by an additive factor of approximately {$\nicefrac{\ell}{10}$}.
\paragraph{Effect of Group Preference Strength.}

An important parameter in the above model is the amount of reward lost when a user is shown items from a group that they do not prefer. 
We fix $\ell = .25$, and vary $\alpha$ (the reward subtracted when choosing an arm from a non-preferred group) from {0 to .25}. 
We note that even for very small values, e.g., $\alpha = .05$, algorithms such as \RAN attain significantly less reward, while $\FAIR$ and $\EPS$ are just slightly worse than the unconstrained (and hence unfair) algorithm. 
As before, we note that $\FAIR$ and $\EPS$ perform almost as well as the unattainable optimum, and are only worse than the unconstrained (and hence unfair) algorithm by an additive factor of approximately {$\nicefrac{\alpha}{4}$}.

\subsection{Experiments on Real-World Data}

\paragraph{Dataset.} We consider the YOW dataset \cite{zhang2005bayesian} which contained data from a collection of 24 paid users who read 5921 unique articles over a 4 week time period. The dataset contains the time at which each user read an article, a [0-5] rating for each article read by each user, and (optional) user-generated categories of articles viewed. We use this data to construct reward distributions for several different contexts (corresponding to different types of users) on a larger set of arms (corresponding to different articles with varying quality) of different types of content (corresponding to groups) that one can expect to see online.

We first created a simple ontology to categorize the 10010 user-generated labels into a total of $g = 7$ groups of content: Science, Entertainment, Business, World, Politics, Sports, and USA. We then removed all articles that did not have a unique label, in addition to any remaining articles that did not fit this ontology. This left us with 3403 articles, each belonging to a single group. We removed all users who did not view at least 100 of these articles; this left 21 users. We think of each user as a single context. We observe that on average there are $k=81$ unique articles in a day, and take this to be the number of ``arms'' in our experiment. The number of articles $k_i$ in group $G_i$ is simply the average number of unique articles observed from that group in a day. 
Lastly, we note that  the articles suggested to users in the original experiment were selected via a recommendation system tailored for each user. We note that some users rarely, if ever, look at certain categories at any point in the 4 weeks; the difference in the \% of User Likes (see {Table~\ref{tab:data}}) suggests that some amount of polarization, either extrinsic or intrinsic, is present in the data. This makes it an interesting dataset to work with, and is, in a sense, a worst-case setup for our experiment -- we take the pessimistic view that presenting content from un-viewed categories gives a user 0 reward (see below), yet we attempt to enforce fairness by presenting these categories anyway.

\paragraph{Experimental Setup.} Let $\rho_a$ be the rating of article $a \in [k]$, given by averaging all of the ratings it received across all users; we consider this to be the underlying quality of an article.
The quality of our arms are determined as follows: for a group $i$ with $|G_i|$ articles, we split the articles into $k_i$ buckets containing the $|G_i|/k_i$ lowest rated articles, 2nd lowest, and so on; call each such bucket $G_i^h$ for $h \in [k_i]$. Then, we let the score of arm $h \in [k_i]$ in group $i$ be 
$\rho_{i,h} = \frac{\sum_{a \in G_i^h} \rho_a}{|G_i^h|}.$
This gives the underlying arm qualities for our simulations. 
To determine the user preferences across groups, 
we first calculate the probability of an article being of category \emph{i}, given that a given user $u$  has read the article; formally, 
$q_i^u = \P{ a \in G_i | \mbox{user $u$ viewed $a$}}.$ 
We let the \emph{average reward} of arm $h$ in group $G_i$ for a user $u$ be $\left({\mu^\star}\right)_{i,h}^u = q_i^u\cdot \rho_{i,h}.$
We normalize the average rewards for each user to lie in [0.1,0.9], and assume that when we select an arm we receive its average reward plus 0-mean noise drawn from a truncated Normal distribution $\overline{\mathcal N}(0,0.05)$ which is truncated such that the rewards always lie within $(0,1)$. 

\begin{table}
	
\centering
	\scalebox{1}{
		\begin{tabular}{ |c|c|c|c|c|c| } 
			\hline
			\textbf{Category} & \textbf{\# Ratings} & \textbf{\# Articles/Day} & \textbf{Avg. Rating} & \textbf{\% Users Like} \\ 
			\hline
			\hline
			Science      & 1708&26&3.64&90.5  \\
			\hline
			Entertainment& 1170&18&3.31&76.2  \\
			\hline
			Business     &  847&12&3.64&90.5 \\
			\hline
			World        &  828&12&3.54&76.2  \\
			\hline
			Politics     &  492&7&3.55&47.6  \\
			\hline
			Sports       &  227&3&3.59&14.3  \\
			\hline
			USA           &  227&3&3.48&28.6  \\
			\hline
			\hline
			\textbf{Total}  & \textbf{5501}&\textbf{81}&\textbf{3.54} &-\\
			\hline
		\end{tabular}
		}
		\caption{An overview of the dataset and resulting parameters used in our experiment. 
		We report the average number of unique articles each category has in a day across all users; in our experiment this is equivalent to the number of arms in each category. Lastly, we say that a user \emph{likes} a category if at least 5\% of the articles they read are from that category, and we report the \% of users who like each category.  
	}
		\label{tab:data}
\end{table}

We have 21 users, each of which we think of as a different context. 
We report the normalized cumulative reward averaged across all users for each of the algorithms described above.
Error bars depict the standard error of the mean.

\begin{figure*}[t!]
	\centering
	\begin{tabular}{ccc}
	{\includegraphics[width=0.45\textwidth]{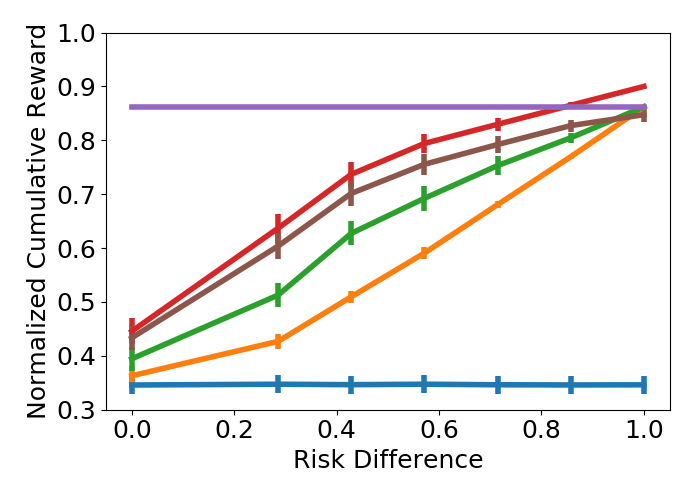}}&
	{\includegraphics[width=0.44\textwidth]{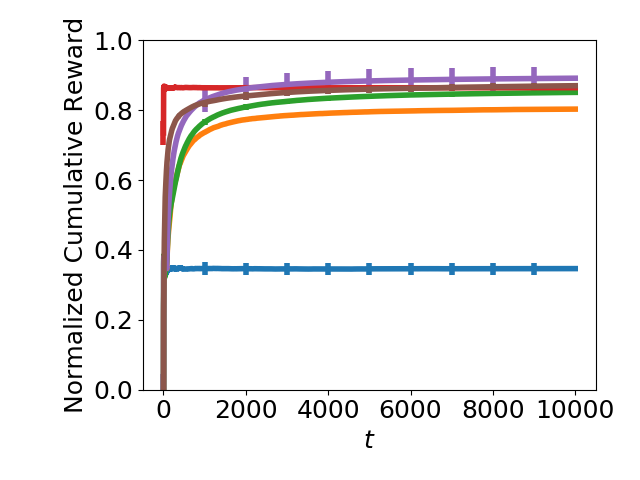}}&		
	{\includegraphics[width=0.1\textwidth,trim = 0cm -4cm 0cm 0cm]{legend_ICML.JPG}}
		\\ 
		$\qquad\;\;$(a) &    $\qquad\;\; $ (b)
	\end{tabular}
	\caption{{Empirical results on Real-World Data.} (a) {\em Tradeoff between Fairness and Reward.} The $x$ axis depicts the fairness of the constrained algorithms as measured by risk difference, and its effect on the normalized cumulative reward is reported. 
 Achieves the same risk difference by instead varying the upper bounds $u_i = u$ and leaving the lower bounds unconstrained $\ell_i = 0$ for all $i$.  (b) {\em Convergence Over Time.} We observe that, for sufficiently many iterations, the normalized cumulative rewards from \FAIR and \EPS converge to \OPT, but \RAN does not. 
	}
	\label{fig:realworld}

\end{figure*}

We consider the risk difference fairness metric (see Section \ref{sec:metrics-app}), and study how guaranteeing a fixed amount of fairness affects the normalized cumulative reward. Note that, if $\ell_i = \ell$ and $u_i = u$ for all $i$, then the risk difference is upper bounded by $u - \ell$, where 1 is the most unfair (users can be served completely different groups) and 0 is the most fair (all users see the same proportion of each group). Thus, for any fixed value $x$ of the risk difference, any $u-\ell = x$ satisfies the guarantee. 
We consider the extreme setting where $\ell = 0$ and we vary only the upper bound $u$ in order to satisfy the desired amount of fairness (see Figure \ref{fig:realworld}).\footnote{Note that in order to compute the risk difference, one must take $\ell = \max\{0,1-(g-1)\cdot u\}$, i.e., one must consider the implicit lower bound implied by the set of upper bounds. Similarly, $u = \min\{1, 1-(g-1)\cdot \ell\}$. The risk difference is guaranteed to be at most $u - \ell$ given these implicit definitions. }

We observe that across the board \FAIR and \EPS outperform \RAN. However, for the tightest constraints, none give much of an advantage over the \NAIVE algorithm. As discussed above, this is due to the pessimistic nature of our reward estimates where we assume that if a viewer does not prefer a given category, they receive 0 reward from viewing such an article. Hence, when enforcing strict constraints, the algorithm must take many such 0-reward decisions. However, for more moderate constraints the performance improves significantly as compared to \NAIVE. Furthermore, this is where the advantage of \EPS and \FAIR as opposed to \RAN can be seen; in effect, \EPS and \FAIR optimize over the (non-zero) sub-optimal groups that the fairness constraints dictate, giving them an advantage over the other randomized approaches to satisfying fairness constraints. 

Unlike the synthetic data, there is now a gap between the performance of the \FAIR, \EPS  and \OPT algorithms; this is largely due to the fact that, with 81 arms, there is a higher learning cost. 
Indeed, even in the completely unconstrained case, we see that \OPT outperforms \UNC by an additive gap of approximately $0.05$. We now compare the empirical regret and running time for \FAIR and \EPS.

\paragraph{Empirical Regret.}We note that \EPS outperforms \FAIR for most settings with constraints. This demonstrates the worse dependence of regret on $k$ for \FAIR. While the regret of \FAIR grows as $O\left(k^3\right)$ with the number of arms $k$, regret of \EPS grows as $O\left(k\right)$. In the case of 81 arms, the difference between \FAIR and \EPS becomes apparent.
Given enough time, the normalized reward of \FAIR does converge to \OPT as predicted by {Theorem \ref{theorem:fofulreg}} (we report the empirical convergence, for $\ell = 0$ and $u = \nicefrac{6}{7}$, in Figure \ref{fig:realworld}b).

\paragraph{Empirical Running Time.}With $\ell = 0$, $u = \nicefrac{6}{7}$, and $T=2000$, the empirical running time of \EPS is 3.9 seconds, whereas the empirical running time of \FAIR is 187.7 seconds, which shows the efficiancy gains we have with $\epsgreedy$ as compared to $\foful$.

Lastly, we observe that satisfying the same risk difference using upper bound constraints as opposed to lower bound constraints results in higher reward when the constraints are not tight. This occurs, again, because of the pessimistic reward calculations -- upper bound constraints allow more probability mass to be kept away from the 0-reward groups for longer while satisfying the same fairness guarantee. 
In general, for a given fairness metric, there could be multiple ways of setting constraints to achieve the same fairness guarantee; an important open question that remains is how to optimize over these different possibilities.

\section{Proofs}
\subsection{Proof of Theorem \ref{theorem:fofulreg}}
\label{sec:proofs_fofulreg}
In this section we prove Theorem \ref{theorem:fofulreg}.
In fact we prove the following more precise version of it.
Since the bound is the same for any context $s \in \mathcal{S}$, we omit the context henceforth.
\begin{theorem}\label{thm:main2} Given the description of $\mathcal{C}$ and the sequence of rewards drawn from a $O(1)$-subgaussian distribution with the expectation vector  $\mu^\star$. 
Assume that $\norm{\mu^\star}_2 \leq \sigma$ for some $\sigma \geq 1$.
Then, with probability at least $1-\delta$, the regret of \foful after time $T$ is:
\begin{align*}
	&{\sf FairRegret}_T \leq \frac{8k\sigma^2}{\gamma}\bigg(\log T + (k-1)\log\frac{64 \sigma^2}{\gamma^2} +\\
	&  2(k-1)\log\big(k\log\big(1 + \nicefrac{T}{k}\big) + 2\log\left(\nicefrac{1}{\delta}\right)\big)+ 2\log\left(\nicefrac{1}{\delta}\right)\bigg)^2.
\end{align*}
\end{theorem}

\paragraph{Notations for the proof.} 
For a positive definite matrix $A\in\bR^{k\times k}$, the weighted $1$-norm and $2$-norm of a vector $x \in \bR^k$ is defined by 
$$\norm{x}_{1,A} := \sum_{i = 1}^k\abs{A^{\nicefrac{1}{2}}x}_i \mbox{and} \; \; \norm{x}_{2,A} := \sqrt{x^\top A x}.$$
Let  $p^\star := \argmax _{p \in \mathcal{C}} \langle \mu^\star, p\rangle$. 
Let the instantaneous regret $R_t$ at time $t$ of \foful be defined as the difference between the expected values of the reward received for $p^\star$ and the chosen arm $p^t$: 
$$R_t= \ip{\mu^\star}{p^\star-p^t}.$$
The cumulative regret until time $T$, ${\sf FairRegret}_T$, is defined as $\sum_{t = 1}^{T}R_t$.
Recall that $r^t$ is the reward that the algorithm receives at the $t$-th time instance. 
Note that the expected value of reward $r^t$ is $\ip{\mu^\star}{p^t}$. 
Let  $$\eta_t := r^t - \ip{\mu^\star}{p^t}.$$
The fact that $r^t$ is $O(1)$-subgaussian implies that $\eta_t$ is also $O(1)$-subgaussian.
Finally, recall that we denote our estimate of $\mu^\star$ at the $t$-th iteration by $\hat{\mu}_t$.

\paragraph{Technical lemmas.}
Towards the proof of Theorem \ref{thm:main2}, we need some results
from \cite{dani2008stochastic} and \cite{YA2011} that we restate in our setting.
The first is a theorem from \cite{YA2011} which helps us to prove that $\mu^\star$ lies in the confidence set $B_t^1$ at each time-step with high probability.
\begin{theorem}[Theorem 2 in \cite{YA2011}] Assume that the rewards are drawn from an $O(1)$-subgaussian distribution with the expectation vector $\mu^\star$.
	Then, for any $0 < \delta < 1$, with probability at least $1-\delta$, for all $t \geq 0, \mu^\star$ lies in the set 
	\begin{align*}
	B_t^2 := \left\{\mu \in \bR^k : \norm{\mu - \hat{\mu}_t}_{2, V_t} \leq \sqrt{\beta_t(\delta)}\right\}
	\end{align*}
	where $\beta_t$ is defined in Step 5 of the \foful algorithm.
	\label{theorem:YAconf}
\end{theorem}
\noindent
As a simple consequence of this theorem we prove that $\mu^\star$ lies inside $B_t^1$ with high probability.
\begin{lemma} $\mu^\star$ lies in the confidence set $B_t^1$ with a probability at least $1-\delta$ for all $t \in T$.
	\label{lemma:confidenceset}
\end{lemma}
\begin{proof}
	\begin{align*}
	\norm{\mu^\star - \hat{\mu}_t}_{1, V_t}\leq \sqrt{k}\norm{\mu^\star - \hat{\mu}_t}_{2, V_t}
	\leq \sqrt{k\beta_t(\delta)},
	\end{align*}
	Here, the first inequality follows from Cauchy-Schwarz and the second inequality holds with probability at least $1-\delta$ for all $t$ due to Theorem \ref{theorem:YAconf}.
\end{proof}

\noindent
The following four lemmas would be required in the proof of our main theorem.
\begin{lemma}[Lemma 7 in \cite{dani2008stochastic}] For all $\mu \in B_t^1$ (as defined in Step 6 of \foful) and all $p \in \cC$, we have: 
	$$\abs{(\mu-\hat\mu_t)^\top p} \leq \sqrt{k\beta_t(\delta)p^\top V_t^{-1}p}$$ where $V_t$ is defined in Step 10 of the \foful algorithm.
	\label{lemma:dani7}
\end{lemma}
\begin{lemma}[Lemma 8 in \cite{dani2008stochastic}]  If $\mu^\star \in B_t^1$, then $$R_t \leq 2\min\left(\sqrt{k\beta_t(\delta)p^{t\top} V_t^{-1}p^t}, 1\right).$$
	\label{lemma:dani8}
\end{lemma}
\begin{lemma}[Lemma 11 in \cite{YA2011}] Let $\left\{p^t\right\}_{t=1}^{T}$ be a sequence in $\bR^k$, $V = \cI_k$ be the $k\times k$ identity matrix, and define $V_t := V + \sum_{\tau=1}^{t}p^{\tau}p^{\tau\top}$. Then, we have that:
	\begin{align*}
	\log\det(V_T) \leq \sum_{t = 1}^{T}\norm{p^t}_{V_{t-1}^{-1}}^2 \leq 2\log\det(V_T).
	\end{align*}
	\label{lemma:ya11}
\end{lemma}
\noindent Finally, we state another result from the proof of Theorem 5 in \cite{YA2011}.
\begin{lemma} For any $T$, we have the following upper bound on the value of $\frac{\beta_T(\delta)}{\gamma}\log\det(V_T)$:
	\begin{align*}
	&\frac{\beta_T(\delta)}{\gamma}\log\det(V_T) \leq \frac{\sigma^2}{\gamma}\bigg(\log T + (k-1)\log\frac{64\sigma^2}{\gamma^2}\\
	&+2(k-1)\log(k\log\left(1 + \nicefrac{T}{k}\right) + 2\log(\nicefrac{1}{\delta}))+ 2\log(\nicefrac{1}{\delta})\bigg)^2
	\end{align*}
	where $\gamma$ is defined as in Equation \ref{eq:gamma}, $\sigma$, $\delta$ are as in Theorem \ref{thm:main2}.
	\label{lemma:appendixE}
\end{lemma}

\paragraph{Proof of Theorem \ref{thm:main2}.}
Start by noting that 
\begin{align}
{\sf FairRegret}_T=\sum_{t=1}^T R_t \leq \sum_{t=1}^T \frac{R_t^2} {\gamma}. \tag{$0 < \gamma \leq R_t$}
\end{align}
From Lemma \ref{lemma:confidenceset}, we know that $\mu^\star$ lies in $B_t^1$ with probability at least $1-\delta$. 
Hence, with probability at least $1-\delta$, we have:
\begin{align*}
\sum_{t=1}^T\frac{R_t^2}{\gamma} &\leq \sum_{t=1}^T\frac{4k\beta_t(\delta)}{\gamma}\norm{p^t}_{2, V_t^{-1}}^2 \tag{from Lemma \ref{lemma:dani8}}\\   
&\leq \frac{4k\beta_T(\delta)}{\gamma}\sum_{t=1}^T\norm{p^t}_{2, V_t^{-1}}^2 \tag{since $\beta_t$ is increasing with $t$}\\
&\leq \frac{8k\beta_T(\delta)}{\gamma}\log\det(V_T) \tag{from  the second ineq. in Lemma \ref{lemma:ya11}}.
\end{align*}
To conclude the proof of the theorem, we combine the bound on ${\sf FairRegret}_T$ with  Lemma \ref{lemma:appendixE} to give us the required upper bound on the RHS of the last inequality above:
\begin{align*}
&\sum_{t=1}^TR_t \leq \frac{8k\sigma^2}{\gamma}\bigg(\log T + (k-1)\log\frac{64\sigma^2}{\gamma^2}+\\
&  2(k-1)\log(k\log\left(1 + \nicefrac{T}{k}\right) + 2\log(\nicefrac{1}{\delta}))+ 2\log(\nicefrac{1}{\delta})\bigg)^2.
\end{align*}

\subsection{Proof of Theorem \ref{theorem:epsgreedy}}
\label{sec:proofs_epsgreedy}
Next, we give the proof of Theorem \ref{theorem:epsgreedy}.
\begin{proof}
	Let $v^\star = [v^\star_1,\cdots,v^\star_k]$ be the optimal probability.
	Conditioned on the history at time $t$, the expected regret of the $\epsgreedy$ at iteration $t$ can be bounded as follows 
	\begin{align*}
	R(t) &= \mu^{\star \top} v^\star - \left( (1-\epsilon_t)\mu^{\star \top} \bar{v}^t + \frac{\epsilon_t}{k}\sum_{a=1}^{k}\mu^\star_a \right) \\ &
	\leq (1-\epsilon_t) \mu^{\star \top} (v^\star - \bar{v}^t) + \mu^{\star \top}v^\star\epsilon_t \\&
	\leq   (1-\epsilon_t) \mu^{\star \top}v^\star 1\{\bar{v}^t \neq v^\star \} +\mu^{\star \top}v^\star \epsilon_t. 
	\end{align*}
	Let $n= \nicefrac{4}{(\eta d^2)}$. For $t\leq n$ we have $\epsilon_t=1$. The expected regret of the $\epsilon$-greedy is 
	\begin{align} \label{eq:regret_greedy}
	& \mathbb{E}\left[{\sf FairRegret}_T\right]  \leq \nonumber \\& \mu^{\star \top}v^\star \sum_{t=n+1}^{T} \mathbb{P}\{\bar{v}^t \neq v^\star \} + \mu^{\star \top}v^\star\sum_{t=1}^{T} \epsilon_t.
	\end{align}
	Let $\Delta \mu = \bar{\mu} - \mu^\star$. 
	Without loss of generality, let $\mu^{\star \top} v_i > \mu^{\star \top} v_j$ for any $v_i,v_j \in V(C)$ with $i<j$. 
	Hence, $v_1 = v^\star$. Let $\Delta_i = \mu^{\star \top} (v_1 - v_i)$.
	As a result $\Delta_2 = \gamma$.
	The event $\bar{v}^t \neq v^\star$ happens when $\bar{\mu}^\top_t v_i > \bar{\mu}^\top_t v_1$ for some $i>1$. 
	That is $(\mu^\star + \Delta \mu_t)^\top(v_i-v_1) = -\Delta_i + \Delta \mu_t^\top(v_i-v_1)\geq0$. As a result, we have
	\begin{align}
	\mathbb{P}\{\bar{v}^t \neq v^\star \} &= \mathbb{P}\{\bigcup_{v_i \in V(C) \backslash v_1  } \Delta \mu_t^\top(v_i-v_1) \geq \Delta_i \}  \nonumber\\&
	\leq \mathbb{P}\{\bigcup_{v_i \in V(C) \backslash v_1  } \|\Delta \mu_t\|_\infty \|v_i-v_1\|_1 \geq \Delta_i \} \label{eq:temp}\\&
	\leq \mathbb{P}\{\bigcup_{v_i \in V(C) \backslash v_1  } \|\Delta \mu_t\|_\infty \geq \frac{\Delta_i}{2} \} \nonumber\\&
	= \mathbb{P}\{\|\Delta \mu_t\|_\infty \geq \frac{\gamma}{2} \} \nonumber\\&
	= \mathbb{P}\{\bigcup_{j \in[k] } |\Delta \mu_{t,j}| \geq \frac{\gamma}{2} \}\nonumber\\&
	\leq \sum_{j\in [k]}\mathbb{P}\{ |\Delta \mu_{t,j} |\geq \frac{\gamma}{2} \} . 
	\end{align}
	In \eqref{eq:temp} we use Holder's inequality.
	Let $E_t = \eta \sum_{\tau=1}^{t} \nicefrac{\epsilon_t}{2} $ and
	let $N_{t,j}$ be the number of times that we have chosen arm $j$ up to time $t$.
	Next, we bound $\mathbb{P}\{| \Delta \mu _{t,j}|\geq \frac{\Delta_2}{2} \}$.
	\begin{align}
	&\mathbb{P}\{ |\Delta \mu_{t,j}| \geq \frac{\gamma}{2} \} \nonumber\\&
	= \mathbb{P}\{ |\Delta \mu_{t,j}| \geq \frac{\gamma}{2} | N_{t,j} \geq E_t\} \mathbb{P}( N_{t,j} \geq E_t) \nonumber \\&+ \mathbb{P}\{ |\Delta \mu_{t,j}|\geq \frac{\gamma}{2} | N_{t,j} < E_t\} \mathbb{P}( N_{t,j} < E_t) \nonumber \\ &
	\leq \mathbb{P}\{ |\Delta \mu_{t,j}|\geq \frac{\gamma}{2} | N_{t,j} \geq E_t\} +
	\mathbb{P}( N_{t,j} < E_t). \label{eq:temp2}
	\end{align}
	As $q_f \in \left\{q : B_{\infty}(q,\eta) \subset \mathcal{C} \right\}$, we have $q_{a,f} > \eta $. 
	Next, we bound each term of \eqref{eq:temp2}. First, using Chernoff-Hoeffding bound we have 
	\begin{equation} \label{eq:temp3}
	\mathbb{P}\{ |\Delta \mu_{t,j}|\geq \frac{\gamma}{2} | N_{t,j} \geq E_t\} \leq 2\exp(-\frac{E_t\gamma^2}{2}).
	\end{equation}
	Second, using Bernstein inequality we have
	\begin{equation} \label{eq:temp4}
	\mathbb{P}( N_{t,j} < E_t) \leq \exp(-\frac{E_t}{5}).	
	\end{equation}
	For $t\leq n$, $\epsilon_t=1$ and $E_t =  \eta t/ 2$. For $t > n$ we have
	\begin{align} 
	E_t = \frac{\eta \cdot n}{2} + \sum_{i=n+1}^{t} \frac{2 }{d^2i} &\geq\frac{2}{d^2} + \frac{2}{d^2}\ln(\frac{t}{n}) \nonumber \\&
	= \frac{2}{d^2}\ln(e\frac{t}{n}).\label{eq:temp5}
	\end{align} 
	By plugging \eqref{eq:temp3}, \eqref{eq:temp4} and \eqref{eq:temp5} in \eqref{eq:temp2} and noting that $\gamma <1/2$ we get
	\begin{align} \label{eq:temp6}
	\mathbb{P}\{ |\Delta \mu_{t,j}|&\geq \frac{\gamma}{2} \} \leq
	\left(\frac{n}{et}  \right)^{\frac{\gamma^2}{d^2}} +
	\left(\frac{n}{et}  \right)^{\frac{4}{10d^2}} \nonumber \\&
	\leq \left(\frac{n}{et}  \right) +
	\left(\frac{n}{et}  \right)^{\frac{4}{10d^2}} \leq 2\left(\frac{n}{et}  \right) .
	\end{align}
	Plugging \eqref{eq:temp6} in \eqref{eq:regret_greedy} yields
	\begin{align} \label{eq:regret_greedy2}
	&\mathbb{E}\left[{\sf FairRegret}_T\right]  \leq \nonumber \\& \mu^{\star \top}v^\star 
	\left(   (1+\frac{2n}{e}) \ln T + n \right).
	\end{align}
	By substituting $n = \nicefrac{4}{(\eta d^2)} $ in the regret above and noting that $\gamma\leq 2d$  we conclude the proof
	\begin{align} \label{eq:regret_greedy3}
	&	\mathbb{E}\left[{\sf FairRegret}_T\right]  \leq \nonumber \\& \mu^{\star \top}v^\star 
	\left(   \left(1+\frac{4}{\eta d^2}\right) \ln T + \frac{4}{\eta d^2} \right) = O\left( \frac{\log T}{\eta \gamma^2}  \right).
	\end{align}
\end{proof}

\subsection{Laminar Constraints} 
\label{sec:laminar}

In this section, we consider a laminar type of constraints. Let the Groups $G_1,\ldots,G_g \subseteq [k]$ be such that: $G_i \cap G_j \neq \emptyset$ implies $G_i \subseteq G_j$ or $G_j \subseteq G_i$.

In this case, the linear programming problem can be solved efficiently by a greedy algorithm.
The groups form a tree data structure, where the children are the largest groups that are subset of the parents.
For example in Figure~\ref{fig:Laminar}, the groups $G_1$ and $G_2$ are subsets of the arms $[k]$ and $G_1 \cap G_2 = \emptyset$. Similarly, the groups $G_3$ and $G_4$ are subsets of the group $G_1$ and $G_3 \cap G_4 = \emptyset$. $G_5$ is a subset of $G_2$.

\begin{figure}[t] 
	\centering
	\hspace{-.2in}
	\includegraphics[width=.6\linewidth]{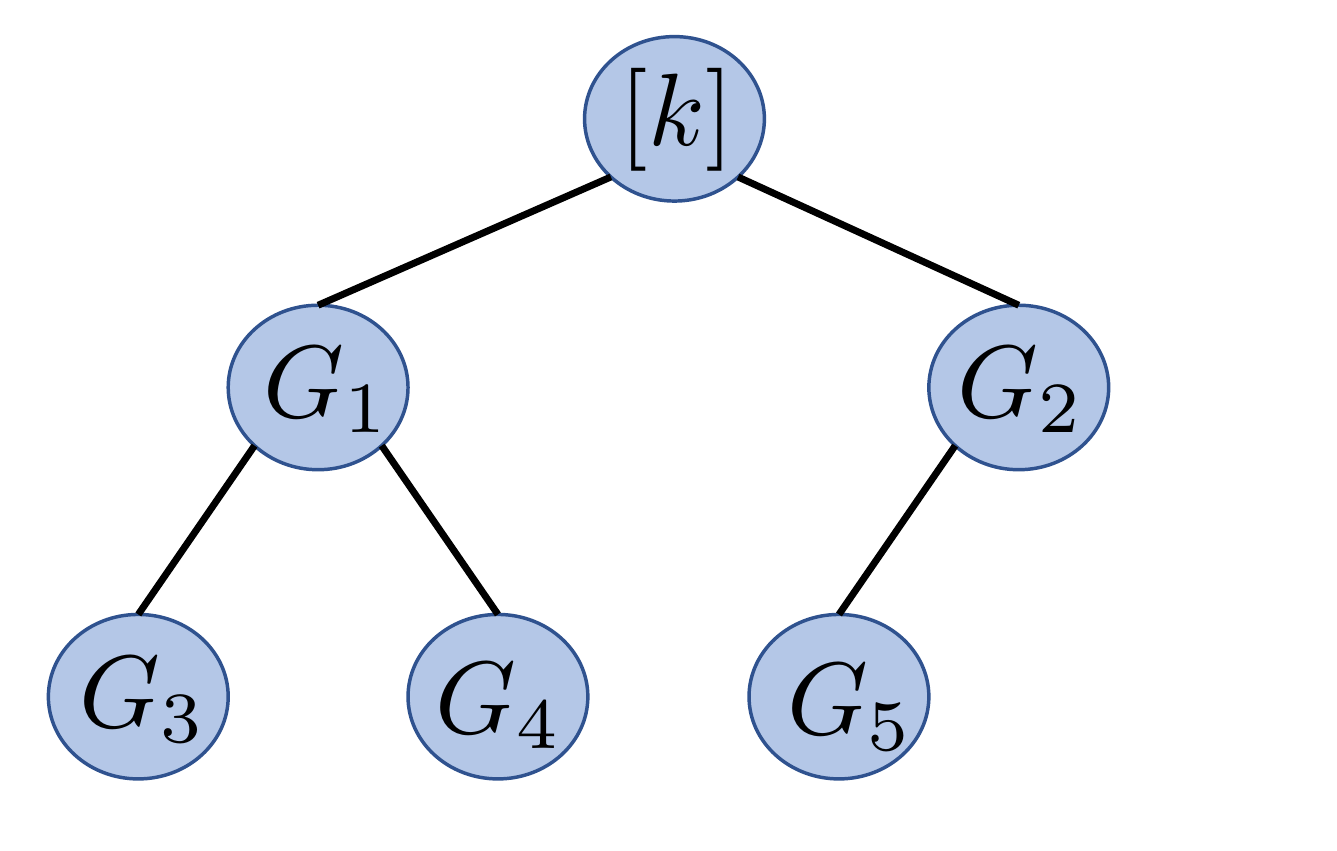} 	
	\caption{Laminar Group structure.}
	\label{fig:Laminar}
\end{figure}

If the lower bound $\ell_i$ for a group $G_i$ is smaller than the sum of the lower bounds for the children groups, then we increase it to the sum of of the lower bounds for the children groups.
Fo example in Figure~\ref{fig:Laminar}, if $\ell_1< \ell_3+ \ell_4$, then we increase it to $\ell_3+ \ell_4$. This is because satisfying the lower bound of $G_3$ and $G_4$ automatically satisfies the probability of $G_1$.
Similarly, if the upper bound $u_i$ for a group $G_i$ is larger than the sum of the lower bounds for the children groups, then we decrease it to the sum of of the upper bounds for the children groups.
For example in Figure~\ref{fig:Laminar}, if $u_1 > u_3+ u_4$, then we decrease it to $u_3+ u_4$. This is again because the total probability that an arm in group $G_i$ is selected cannot be larger than the upper bounds of its children.
This change of the upper and lower bounds does not change the optimum of the LP problem.   

In the greedy algorithm, first we satisfy the lower bounds, then we allocate the remaining probability such that the upper bounds are not violated.

In satisfying the lower bounds, we take a bottom-up approach.
We start from the leaves and satisfy the lower bound by giving the item to the arm $a$ with the largest reward in the group, i.e., $\argmax_{a \in G_i} \mu_a$. In our example, we set the probability of $\argmax_{a \in G_3} \mu_a$ to $\ell_3$,   the probability of $\argmax_{a \in G_4} \mu_a$ to $\ell_4$ and the probability of $\argmax_{a \in G_5} \mu_a$ to $\ell_5$.
Next, we proceed with satisfying the lower bound for the parents. In our example, we add the probability of $\ell_1- (\ell_3+\ell_4)$ to $\argmax_{a \in G_1} \mu_a$, and the probability of $\ell_2- \ell_5$ to $\argmax_{a \in G_2} \mu_a$.
We continue the process until no group remains infeasible. Finally, we assign the remaining probability to the arm with the largest reward.  In our example, we add the probability of $1- (\ell_1+\ell_2)$ to $\argmax_{a \in [k]} \mu_a$.

The remaining probability is first allocated to $\argmax_{a \in [k]} \mu_a$ until we reach the probablility for one of the upper bound constraints. 
Then, we eliminate the arms inside that group, and we allocate some probability to the arm with the maximum reward until another upper bound constraint is reached. 
We continue this process until either our distribution over arms is a probability distribution or we cannot allocate more probability to any arm without violating a constraint.

Let the probability that an arm from the group $G_i$ is selected be $\sum_{a\in G_i} v_a = q_i$ and the children of group $G_i$ be $G_{i1}, G_{i2},\ldots,G_{im}$. We denote the optimal allocation (subject to the constraints) at node $G_i$ be $OPT(G_i,q_i,\ell,u)$. Then, we have
\begin{align} \label{eq:temp_opt}
&OPT(G_i,q_i,\ell,u) = \\&\sum_{j=1}^{m} OPT(G_{ij},\ell_{ij},\ell,1) + OPT(G_i,q_i-\sum_{j=1}^{m}\ell_{ij},0,u). \nonumber
\end{align}
In satisfying the lower bounds (i.e., first term in \eqref{eq:temp_opt})
if we do not change the probability of selecting an arm inside a group $G_i$, i.e., $\sum_{a\in G_i} v_a$, then it does not effect the parent groups, hence we can locally optimize the problem beginning from the smaller groups.

We can show by contradiction that the procedure for allocating the remaining probability (i.e., second term in \eqref{eq:temp_opt}) is optimal.
This is because, if the probability of $\argmax_{a \in [k]} v_a$ can be increased without violating an upper bound or we can increase it by reducing another arm's probability of winning, then the current probability allocation is not optimal.

The running time of the algorithm is linear in the number of the arms $k$ and height of the tree.
Given that height of the tree is less than $g$, the total running time becomes $O(gk)$.

\section{Conclusion}

In this paper we initiate a formal study of incorporating fairness in bandit-based personalization algorithms.
We present a  general framework  that allows one to ensure that a fairness metric of choice attains a desired fixed value by providing appropriate upper and lower bounds on the probability that a given group of content is shown. 
We present two new bandit algorithms that perform well with respect to reward, improving the regret bound over the state-of-the-art. 
\EPS is particularly fast and we expect it to scale well in web-level applications.
Empirically, we observe that our \FAIR and \EPS algorithms indeed perform well; they not only converge quickly to the theoretical optimum, but this optimum, even for the tightest constraints (which attain a risk difference of 0) on the pessimistic arm values selected, is within a factor of 2 of the unconstrained rewards.

From an experimental standpoint, it would be interesting to explore the effect of the group structure in conjunction with the rewards structure on the tradeoff between fairness and regret.
Additionally, testing this algorithm in the field, in particular to measure user satisfaction given diversified news feeds, would be of significant interest. 
Such an experiment would give deeper insight into the benefits and tradeoffs between personalization and diversification of content, which could then be leveraged to set the appropriate rewards and parameters in our algorithm.

\bibliographystyle{plain}
\bibliography{references} 


\appendix

\section{Discrimination Metrics}
\label{sec:metrics-app}

A wide variety of discrimination metrics, in particular with relation to classification, have been studied in the literature. 
In this section we take a look at some such prevalent metrics and show that the constraints are general enough to encapsulate them. 

Formally, fairness metrics are defined for a sets of items or people $[k]$, and a group $G \subsetneq [k]$ of items that have a sensitive attribute (e.g., a particular gender or ethnicity). 
Now, consider any classifier $\chi: \mathcal [k] \to \{1, 0\}$. {In general, one thinks of $\chi(\cdot) = 1$ as being a ``positive'' outcome (i.e., a loan is given or bail approved).}
Fairness metrics measure how the presence or absence of a sensitive attribute affects the probability of a positive outcome.
To translate this to our setting, $[k]$ is the set of content, and $G_i \subsetneq [k]$ is a protected group, e.g., news articles about a minority opinion, ads for high-paying jobs, or paid content on a news site. 
Now, let $\chi_p: [k] \to  \{1, 0\}$ be a probabilistic classifier defined by a distribution $p \in [0,1]^k$ such that for any $a \in [k]$, we have $\chi_p(a) = 1$ with probability $p_a$ and $0$ otherwise (thus, $\mathbb E[\chi_p(a)] = p_a$).
In effect, this classifier encodes the ``positive'' outcome of being selected given distribution $p$.
Fairness metrics can now be defined for this classifier, and we give some examples below.

The $x$\% rule (often 80\% in legal rulings \cite{DisparateImpactBook})  
for measuring disparate impact states that the ratio between the percentage of items having a certain sensitive attribute value assigned the positive decision outcome and the percentage of items not having the value also assigned the positive outcome should be at least $\nicefrac {x}{100}$ (see, e.g., \cite{pmlr-v54-zafar17a, feldman2015certifying}). 
This metric can be translated to our setting as follows: 
Given $x$, we say that a distribution is $x$\% fair with respect to group $i$ if
\[ \frac{x}{100} \leq 
\frac{\sum_{a \in G_i} p_a } {\sum_{a \not\in G_i} p_a  }.\]
Thus, a content selection algorithm could be deemed $x$\% fair if the distribution $p^t$ over content is $x$\% fair with respect to all groups $i$ for all $t$.
Given $x$, if we set $\ell_i=\frac{x}{100+x}$, and $u_i$ to be any value more than $\ell_i$, the constraints of the form 
$   \ell_i \leq \sum_{a \in G_i} p_a^t \leq u_i $
ensure that 
$$\frac{x}{100} \leq \frac{\sum_{a \in G_i} p_a } {\sum_{a \not\in G_i} p_a  },$$ implying  $x$\% content-fairness. 
In fact, our framework also allows us to have a different $x$ for each group.

Many other fairness metrics can be defined with respect to the probability of a negative outcome for the group with the sensitive attribute ({$\nu_G$}), the group without the sensitive attribute ({$\nu_{\neg G}$}) and the overall dataset ({$\nu$}); e.g., \cite{HajianDiscriminationPrevention2013, HajianPrivacy2014,Hajian2014generalization}. 
Examples include \emph{selection lift} ($s_{\mbox{lift}} =  \frac{\nu_{G}}{\nu_{\neg G}}$), \emph{extended lift} (${e_{\mbox{lift}}} = \frac{\nu_{G}}{\nu}$), and \emph{odds lift} (${o_{\mbox{lift}}} = \frac{\nu_{G}(1-\nu_{\neg G})}{\nu_{\neg G}(1-\nu_{G})}$). 
We say that the classifier $C_p$ is $\beta-{\mbox{lift}}$-fair with respect to group $i$, or more generally, the distribution $p$ is   $\beta-{\mbox{lift}}$-fair with respect to group $i$ if ${\mbox{lift}} \leq \beta$ (for any of the above definitions of $\cdot_{\mbox{lift}}$). 
Given the constraints  
$    \ell_i \leq \sum_{a\in G_i}p_a^t  \leq u_i, $

$\beta$-lift fairness can be ensured for selection lift, if we set $u_i,\ell_i$ such that 

$ \frac{u_i}{\ell_i} \leq \beta$
for all $i$. 
Similarly, appropriate conditions  can be derived for extended and odds lift.

A different class of discrimination metrics measure the additive as opposed to multiplicative functions of the values $\nu_{G}, \nu_{\neg G}$, and $\nu$ defined above (see, e.g., \cite{Calders2010},\cite{Ruggieri:2014:UTA:2870621.2870623}). For example, {\em risk difference} ($RD:= \nu_{G} - \nu_{\neg G}$) and {\em extended difference} (${ED} := \nu_{G} - \nu$). 
Given the constraints, we have $\{RD, ED\} \leq (u_i-\ell_i)$ for all $i$. Thus, given $\beta>0$, we can ensure that their value is always less than $\beta$ by setting the appropriate values of $\ell_i$ and $u_i$.
We explore the effect on the cumulative reward of guaranteeing a fixed value of $RD$ in Section~\ref{sec:experiments}.

\end{document}